\documentclass{article}

\usepackage{PRIMEarxiv}

\usepackage[utf8]{inputenc} 
\usepackage[T1]{fontenc}    
\usepackage{hyperref}       
\usepackage{url}            
\usepackage{booktabs}       
\usepackage{amsfonts}       
\usepackage{nicefrac}       
\usepackage{microtype}      
\usepackage{lipsum}
\usepackage{fancyhdr}       
\usepackage{graphicx}       
\graphicspath{{media/}}     

\usepackage{microtype}
\usepackage{graphicx}
\usepackage{subfigure}
\usepackage{booktabs} 

\usepackage{hyperref}
\usepackage{adjustbox}
\usepackage{algorithm}
\usepackage{algpseudocode}
\usepackage{balance}
\usepackage{amsmath}
\usepackage{amssymb}
\usepackage{mathtools}
\usepackage{amsthm}
\usepackage{multirow}
\usepackage{lipsum}

\usepackage{cite}
\usepackage{hyperref}

\usepackage{footnote}
\usepackage{etoolbox}
\makeatletter
\patchcmd{\@makecaption}
  {\scshape}
  {}
  {}
  {}
\makeatother
\usepackage{xcolor}

\usepackage{amsfonts}       
\usepackage{nicefrac}       

\newtheorem{corollary}{Corollary}

\usepackage{bm}
\usepackage{amsmath}

\newtheorem{theorem}{Theorem}

\usepackage{graphicx,float} 
\usepackage{amssymb}
\usepackage{algorithm}
\usepackage{amsmath,bm}
\usepackage{makecell}

\usepackage{hyperref}
\usepackage{booktabs}
\usepackage{lipsum}
\usepackage{multicol}
\usepackage[capitalize,noabbrev]{cleveref}

\theoremstyle{plain}
\theoremstyle{definition}
\theoremstyle{remark}

\title{Joint covariate-alignment and concept-alignment: a framework for domain generalization }

\author{
  Thuan Nguyen \\
 Department of Computer Science \\ Tufts University \\ Medford, MA 02155 \\
  \texttt{\{Thuan.Nguyen@tufts.edu} \\
   \And
  Boyang Lyu \\
  Department of Electrical and Computer Engineering \\ Tufts University \\ Medford, MA 02155\\
  \texttt{Boyang.Lyu@tufts.edu} \\
     \And
  Prakash Ishwar \\
  Department of Electrical and Computer Engineering \\ Boston University \\ Boston, MA 02215\\
  \texttt{pi@bu.edu} \\
  \And
    Matthias Scheutz \\
 Department of Computer Science \\ Tufts University \\ Medford, MA 02155 \\
  \texttt{\{Matthias.Scheutz@tufts.edu} \\
     \And
  Shuchin Aeron \\
  Department of Electrical and Computer Engineering \\ Tufts University \\ Medford, MA 02155\\
  \texttt{Shuchin@ece.tufts.edu}
}

\newif\ifcomments
\commentstrue
\commentsfalse

\ifcomments
    \usepackage{ulem}
    \def\saedit#1{{$\!$\color{blue} [SA: #1]}}

    \def\tnedit#1{{$\!$\color{red} [TN: #1]}}
    \def\tnrep#1#2{{\color{red} \sout{#1}}{\color{blue}#2}}
    \def\tncomment#1{{$\!$\color{green} [TN: #1]}}
    \def\picomment#1{{$\!$\color{orange} [PI: #1]}}
    \def\piedit#1#2{{\color{red} \sout{#1}}{\color{blue}#2}}
\else 
    \def\saedit#1{}

    \def\tnedit#1{}
    \def\tncomment#1{}        
    \def\tnrep#1#2{#2}
    \def\picomment#1{}
    \def\piedit#1#2{#2}
\fi

\begin{document}
\maketitle

\begin{abstract}
In this paper, we propose a novel domain generalization  (DG) framework based on a new upper bound to the risk on the unseen domain. Particularly, our framework proposes to jointly minimize both the covariate-shift as well as the concept-shift between the seen domains for a better performance on the unseen domain. While the proposed approach can be implemented via an arbitrary combination of covariate-alignment and concept-alignment modules, in this work we use well-established approaches for distributional alignment namely, Maximum Mean Discrepancy (MMD) and covariance Alignment (CORAL), and use an Invariant Risk Minimization (IRM)-based approach for concept alignment. Our numerical results show that the proposed methods perform as well as or better than the state-of-the-art for domain generalization on several data sets\footnote{This paper was accepted at 32nd IEEE International Workshop on Machine Learning for Signal Processing (MLSP 2022).}.
\end{abstract}


\section{Introduction}
\label{sec: introduction}
Domain generalization (DG) has been studied extensively over the past decade as an important practical problem arising in a number of areas such as computer vision, signal processing, and medical imaging \cite{wang2022generalizing} \cite{zhou2021domain}. Like standard learning settings, DG  aims to learn a model from several seen domains (training data) that can generalize well on an unseen domain (test data). However, in contrast to the standard setting where the test data is assumed to come from the same distribution as training data, in DG, the distribution of the test data is different, \textit{i.e.}, there is a presence of what is referred to as a \emph{distribution shift}. This phenomena can be observed in many practical settings \cite{taori2020measuring}.  

A number of approaches for DG are based on the assumption that there exist domain-invariant features that are transferable and unchanged from domain to domain. Thus, a classifier designed on top of these features will likely generalize well to the unseen domain. In practice, DG methods based on this assumption consist of two steps: (\textit{a}) learning a good representation function that outputs the domain-invariant features,  and (\textit{b}) designing a classifier based on these domain-invariant features. 

\piedit{}{Different definitions of domain-invariant features have been proposed and they lead to different feature selection schemes.} In \tnrep{\cite{kumar2018co, sankaranarayanan2018generate, muandet2013domain, li2018domain, ghifary2016scatter, erfani2016robust, jin2020feature, blanchard2021domain, bui2021exploiting,shen2018wasserstein,lyu2021barycentric},}{\cite{li2018domain, ghifary2016scatter, erfani2016robust, jin2020feature, blanchard2021domain, bui2021exploiting,shen2018wasserstein,lyu2021barycentric},} a feature is defined as domain-invariant if its marginal distribution is unchanged over domains. On the other hand, in \cite{arjovsky2019invariant,li2021invariant,ahuja2021invariance,wang2021respecting, zhao2020domain, salaudeen2021exploiting, lu2021nonlinear,ahuja2020invariant} a feature is considered as domain-invariant if its conditional distribution of the label given the representation feature is unchanged from domain to domain. 

The differences in the definitions are rooted in the different modeling assumptions regarding the domains. Particularly, in the \textit{covariate-shift} setting \cite{kouw2018introduction}, the marginal distribution of data varies across the domains. On the other hand, in the setting of \textit{concept-shift} \cite{kouw2018introduction}, the conditional distribution of the label (class) conditioned on the data varies from domain to domain\footnote{The setting where the conditional distribution of data conditioned on the label (class) changes from domain to domain is also referred to as \textit{concept-shift}. However, we do not consider that scenario in this paper.}. 
We survey related work in this context in Sec.~\ref{sec: related work}.

In this paper, we revisit the seminal upper bound for the \piedit{}{prediction} risk in the unseen domain derived in \cite{ben2007analysis} and show the necessity of jointly designing a representation function that minimizes both the covariate-shift and the concept-shift \piedit{}{risks}.  
We make the following contributions:
\begin{itemize}
    \item We derive a novel upper bound \piedit{}{for the prediction} risk in the unseen domain that \piedit{depends on}{consists of} two terms, namely, a covariate-alignment term and a concept-alignment term. This theoretical result motivates a domain generalization \piedit{framework}{algorithm} that combines both covariate-alignment and concept-alignment algorithms. 
     
    \item We propose two new algorithms MMD-CEM and CORAL-CEM that combine, \piedit{}{respectively, the} CORrelation ALignment algorithm (CORAL) \cite{sun2016deep} and the Maximum Mean Discrepancy algorithm (MMD) \cite{li2018domain} with the Conditional Entropy Invariant Risk Minimization \tnrep{(IRM)}{} algorithm (CEM) \cite{nguyen2022conditional}. 
    
    \item We compare the proposed algorithms (MMD-CEM, CORAL-CEM) with \piedit{a number of competitive}{six competing} \tnrep{methods for DG.}{DG methods.} Our methods \piedit{show}{exhibit} similar or better performance \piedit{in comparison to six compared algorithms}{compared to the six alternatives} on both CS-CMNIST \cite{ahuja2021empirical} and CMNIST \cite{arjovsky2019invariant} datasets. 
\end{itemize}
The remainder of this paper is structured as follows. In Sec.~\ref{sec: related work}, we summarize the recent work on  DG dealing with covariate-shift and concept-shift. Sec.~\ref{sec: problem setup} formally defines the problem and clarifies notation used. Sec.~\ref{sec: main results} provides the main theoretical results which motivate the practical algorithms of Sec.~\ref{sec: practical method}. Finally, we provide the numerical results in Sec.~\ref{sec: numerical results} and conclude in Sec.~\ref{sec: conclusion}.

\section{Related work}
\label{sec: related work}


Under the covariate-shift setting,
domain alignment methods focus on aligning the marginal distributions of the representation from different seen domains by minimizing distributional divergences or distances \tnrep{\cite{kumar2018co, sankaranarayanan2018generate,russo2018source,hu2018duplex, muandet2013domain, li2018domain, ghifary2016scatter, erfani2016robust, jin2020feature, blanchard2021domain, bui2021exploiting,shen2018wasserstein,lyu2021barycentric}}{\cite{li2018domain, ghifary2016scatter, erfani2016robust, jin2020feature, blanchard2021domain, bui2021exploiting,shen2018wasserstein,lyu2021barycentric}}.
For instance, in \cite{shen2018wasserstein,lyu2021barycentric}, the Wasserstein distance between distributions of the seen domains in representation space is minimized. 
In \cite{li2018domain}, Li \textit{et al.} proposed the MMD algorithm to minimize the maximum mean discrepancy distance between seen domain distributions in the representation space. In a similar vein, the CORAL algorithm \cite{sun2016deep} is based on the idea of matching the mean and covariance of feature distributions from different domains. 
In \tnrep{\cite{muandet2013domain,erfani2016robust},}{\cite{erfani2016robust},} the authors use deep neural networks to minimize the difference between the variances of transformed features from seen domains to achieve domain generalization. 
In \cite{bui2021exploiting}, the authors aim to learn the domain-invariant features with marginal distributions unchanged from domain to domain together with the domain-specific features to enhance the generalization performance.

Since the conditional distribution of the label given the representation variable is directly related to the classification performance, there are many works that aim to deal with concept-shift, \textit{i.e.}, minimizing the divergence between the conditional distribution of the labels (concepts) given the representation \cite{arjovsky2019invariant,li2021invariant,ahuja2021invariance,wang2021respecting, zhao2020domain, salaudeen2021exploiting, lu2021nonlinear,ahuja2020invariant}.
In  \cite{arjovsky2019invariant,lu2021nonlinear}, linear and non-linear IRM algorithms are proposed for learning the invariant features such that  their conditional distributions are unchanged \piedit{according to}{across} domains, leading to the construction of an optimal classifier for all seen domains. 
In \cite{nguyen2022conditional}, Nguyen \textit{et al.} \piedit{}{developed a} conditional entropy minimization principle to remove spurious invariant features from the invariant features \piedit{}{learned by the} IRM algorithm \cite{arjovsky2019invariant}. In \cite{li2021invariant}, concept-alignment is achieved via minimizing the mutual information of the label given the representation variable for a given domain. 
In \cite{wang2021respecting}, Wang \textit{et al.} handle concept-shift by directly aligning the conditional distribution within each class regardless of domains via the use of Kullback–Leibler divergence.

To the best of our knowledge, there are only \piedit{several}{two} works \piedit{}{in the DG setting} that simultaneously deal with both the covariate-shift and concept-shift  \cite{nguyen2021domain,guo2021out}\footnote{\piedit{}{But there are} several works that simultaneously handle both covariate-shift and concept-shift under the Domain Adaptation (DA) setting, for example the work in \cite{yang2020class,wen2019bayesian}. However, DA is not considered in this paper.}. 
In \cite{nguyen2021domain}, Nguyen \textit{et al.} proposed an algorithm that is capable of achieving both covariate-alignment and concept-alignment by enforcing the latent representation to be invariant under all transformation functions. Their algorithm, however, requires the use of invertible and \piedit{differential}{differentiable} transformation functions together with a generative adversarial network which is known \piedit{}{to be} hard to train. 
In \cite{guo2021out}, motivated by some examples where using concept-alignment \piedit{algorithm only is not enough}{alone is inadequate} for achieving a good generalization, a covariate-alignment algorithm is heuristically \piedit{incorporated}{combined} with a concept-alignment algorithm to achieve a better out-of-distribution generalization. 
In contrast to \cite{nguyen2021domain,guo2021out}, our work appears to be the first work that uses the upper bound for the \piedit{prediction} risk in the unseen domain to theoretically \piedit{prove}{motivate} the necessity of achieving both covariate and concept alignment in DG.

\section{Problem Formulation}
\label{sec: problem setup}
Let $\bm{x} \in \mathbb{X} \subseteq \mathbb{R}^d$ denote the input data and $y \in \mathbb{Y}$ denote the label. The domain generalization task aims to learn a representation function $f: \mathbb{R}^d \rightarrow \mathbb{R}^{d'}$, $d'\leq d$, followed by a classifier $g: \mathbb{R}^{d'} \rightarrow \mathbb{Y}$ 
\tnrep{}{trained using data} from seen domains $D^{(s)}$, but 
\tnrep{}{generalizes} well to data from an unseen domain $D^{(u)}$. 

Let $\bm{z} = f(\bm{x})$ denote the latent variable induced by input data $\bm{x}$ under the mapping $f$.
\piedit{}{We denote the distributions of data from a seen domain and the unseen domain by $p^{(s)}(\bm{x})$ and $p^{(u)}(\bm{x})$\tnrep{}{,} respectively. The distributions of corresponding latent variables in a seen domain and the unseen domain are denoted by $p^{(s)}(\bm{z})$ and $p^{(u)}(\bm{z})$\tnrep{}{,} respectively. We use $Z$ to denote the latent random variable and $Y$ to denote the label random variable.}

\piedit{}{An optimal labeling function for a given domain and a given representation function $f$ is the labeling function that minimizes the prediction risk in the domain using the given representation function.}
For a given representation function $f$, we denote the optimal labeling functions 
\picomment{Optimal in what sense? Is it exactly what I wrote in the last sentence?}
\tncomment{Yes!}
from representation space to label space in a seen domain and the unseen domain by $l^{(s)}(\bm{z})$ and $l^{(u)}(\bm{z})$ respectively. 
For a given $f$,  the risks of using a classifier $g$ (possibly stochastic) in a seen domain and the unseen domain are defined by:
\begin{equation}
\label{eq: 1}
    \epsilon^{(s)}(g, l^{(s)}) = \mathop{\mathbb{E}}_{\bm{z} \sim p^{(s)}(\bm{z}) } d(g(\bm{z}), l^{(s)}(\bm{z})),
\end{equation}
\begin{equation}
\label{eq: 2}
    \epsilon^{(u)}(g, l^{(u)}) = \mathop{\mathbb{E}}_{\bm{z} \sim p^{(u)}(\bm{z}) } d(g(\bm{z}),l^{(u)}(\bm{z})),
\end{equation}
where $\mathop{\mathbb{E}}[\cdot]$ denotes expectation and $d(\cdot,\cdot)$ is a loss function that captures the mismatch between the classifier $g$ and the optimal labeling functions $l^{(s)}, l^{(u)}$. In addition, \piedit{}{for a given representation function $f$ we define the following quantity}:
\begin{equation}
\label{eq: 3}
     \epsilon^{(s)}(g, l^{(u)}) = \mathop{\mathbb{E}}_{\bm{z} \sim p^{(s)}(\bm{z}) } d(g(\bm{z}),l^{(u)}(\bm{z}))
\end{equation}
\piedit{}{which is the loss of using classifier $g$ to input data from a seen domain $s$ transformed by $f$ and labeled with the optimum labeling function of the unseen domain for representation $f$.}
Since $p^{(u)}(\bm{z})$ and $l^{(u)}(\bm{z})$ depend on $f$, we want to learn a good representation function $f$ together with a  classifier $g$ to minimize the prediction risk in the unseen domain  $\epsilon^{(u)}(g, l^{(u)})$. 

\section{Main results}
\label{sec: main results}

\begin{theorem}
\label{theorem: 1}
\piedit{Assume that}{If} the loss function $d(\cdot,\cdot)$ is  non-negative, \piedit{}{symmetric}  and satisfies the triangle inequality, \piedit{}{then} for a given representation function $f$: 
\begin{eqnarray}
\epsilon^{(u)}(g, l^{(u)}) \! &\leq& \epsilon^{(s)}(g,l^{(s)})  \nonumber\\
&\!+& \int_{\bm{z}}|p^{(u)}(\bm{z}) - p^{(s)}(\bm{z})| \: d(g(\bm{z}),l^{(u)}(\bm{z})) \, d\bm{z} \nonumber\\
&\!+& \int_{\bm{z}} p^{(s)} (\bm{z}) \: d(l^{(u)}(\bm{z}), l^{(s)}(\bm{z})) \, d\bm{z}.  \label{eq: theorem}
\end{eqnarray}
\end{theorem}
\begin{proof}
We have:
\begin{small}
\begin{eqnarray}
\epsilon^{(u)}(g, l^{(u)}) \! &=& \epsilon^{(u)}(g,l^{(u)}) \nonumber\\
&+&  \epsilon^{(s)}(g,l^{(s)}) - \epsilon^{(s)}(g,l^{(s)}) \nonumber\\
&+& \epsilon^{(s)}(g,l^{(u)}) - \epsilon^{(s)}(g,l^{(u)}) \label{eq: vy 99}\\
&\leq& \epsilon^{(s)}(g,l^{(s)}) \nonumber \\
&+& |\epsilon^{(u)}(g,l^{(u)})-\epsilon^{(s)}(g,l^{(u)})| \nonumber\\
&+& \epsilon^{(s)}(g,l^{(u)}) - \epsilon^{(s)}(g,l^{(s)}) \label{eq: vy 100}\\
&\leq& \epsilon^{(s)}(g,l^{(s)})  \nonumber\\
&+& \int_{\bm{z}}|p^{(u)}(\bm{z}) - p^{(s)}(\bm{z})| \: d(g(\bm{z}),l^{(u)}(\bm{z})) \, d\bm{z} \nonumber\\
&+& \int_{\bm{z}} p^{(s)}(\bm{z}) \: d(l^{(u)}(\bm{z}), l^{(s)}(\bm{z})) \, d\bm{z} \label{eq: vy 101}
\end{eqnarray}
\end{small}where (\ref{eq: vy 100}) follows from the reorganization of (\ref{eq: vy 99}) and the fact that $a \leq |a|$, $\forall a$ and (\ref{eq: vy 101}) follows from the definitions of $\epsilon^{(u)}(g,l^{(u)})$, $ \epsilon^{(s)}(g,l^{(s)})$, $\epsilon^{(s)}(g,l^{(u)})$ in (\ref{eq: 1}), (\ref{eq: 2}), (\ref{eq: 3}), the symmetry of $d(\cdot,\cdot)$, and the  triangle inequality $d(g(\bm{z}),l^{(u)}(\bm{z}))-d(g(\bm{z}),l^{(s)}(\bm{z})) \leq d(l^{(u)}(\bm{z}), l^{(s)}(\bm{z}))$.
\end{proof}

The bound  in (\ref{eq: theorem}) characterizes the risk of using a classifier trained on a seen domain but applied to the unseen domain. 
The bound contains three terms: \textit{(a)} the first term captures the risk induced by the classifier on a seen domain \piedit{}{for a given representation function $f$}, \textit{(b)} the second term measures the discrepancy \piedit{of}{between} the marginal distributions \piedit{between}{of} seen and unseen domains \piedit{in latent space}{in the latent space corresponding to representation function $f$}, and \textit{(c)} the third term quantifies the mismatch between optimal classifiers \piedit{in}{for} seen and unseen domains \piedit{}{for the given $f$}. 

Our bound shares some similarities with the 
bound proposed by Ben-David \textit{et al.} \cite{ben2007analysis}\tnrep{}{.} \tnrep{}{Particularly, for any representation function $f$, the following bound holds \cite{ben2007analysis}:} 
%
\begin{small}
\begin{eqnarray}
\label{eq: ben}
     \epsilon^{(u)} (g,l^{(u)}) \! \leq \! \epsilon^{(s)} (g, l^{(s)}) \!+\! d_{\mathcal{H}} \Big(p^{(u)}(\bm{z}), p^{(s)}(\bm{z}) \Big) \!+\! \lambda
\end{eqnarray}\end{small}where $d_{\mathcal{H}}$ is $\mathcal{H}$-divergence \cite{kifer2004detecting} and:
\begin{small}
\begin{eqnarray}
\label{eq: combined risk}
\lambda = \inf_{g} \Big( \epsilon^{(s)}(g,l^{(s)}) + \epsilon^{(u)}(g,l^{(u)}) \Big).
\end{eqnarray}\end{small}\piedit{}{In \cite{ben2007analysis}, the above bound was developed for the DA setting. We have adapted it to the DG setting by identifying ``source'' domains in DA with ``seen'' domains in DG and the ``target'' domain in DA with the ``unseen'' domain in DG.}
\piedit{As seen, while}{While} the first two terms of (\ref{eq: theorem}) and (\ref{eq: ben}) are quite similar, the main difference comes from the way that the third term is optimized \piedit{}{in practice}. 
%
\piedit{}{The practical DA algorithm proposed in \cite{ben2007analysis} ignores optimizing the third term $\lambda$ of the upper bound in (\ref{eq: ben}) since the second term of the infimum in (\ref{eq: combined risk}) cannot be estimated during training. Moreover, the first term of the infimum in (\ref{eq: combined risk}) is already captured in the first term of the upper bound in (\ref{eq: ben}). }
\piedit{}{In the DG setting too, since the unseen domain samples are not available during training, one usually optimizes the bound only over several seen domains as in \cite{nguyen2022conditional,wang2022generalizing,zhou2021domain}.}
In contrast, the third term in (\ref{eq: theorem}) measures the mismatch of optimal classifiers between domains and, therefore, can be practically optimized via designing a representation function $f$ followed by a classifier $g$ that is optimal for all (seen) domains. 
As will be shown later, finding an optimal classifier over all domains encourages the use of concept-alignment algorithms. 
Indeed, designing a classifier that is optimal over all domains is the principle behind the recently proposed IRM algorithm \cite{arjovsky2019invariant}.

It is also worth comparing our bound with the bound proposed by Lyu \textit{et al.} \cite{lyu2021barycentric}. 
Indeed, Lemma 1 in \cite{lyu2021barycentric} extends the work in \cite{ben2007analysis} to establish a new upper bound on the prediction risk in the unseen domain that is a function of both the representation mapping $f$ and the classifier $g$. 
While the first two terms of the bound in \cite{lyu2021barycentric} are quite similar to our first two terms, their third term is a constant which is completely independent of both $f$ and $g$. 
Thus, the bound in \cite{lyu2021barycentric} does not encourage the use of concept-alignment algorithms. In practice, the proposed algorithms in \cite{ben2007analysis}  and \cite{lyu2021barycentric} only aim to achieve covariate-alignment, \textit{i.e.}, minimizing the discrepancy of marginal distributions between domains.

Next, to explicitly show that Theorem \ref{theorem: 1} motivates the combination of both covariate-alignment and concept-alignment algorithms, we define the following two terms:
\begin{eqnarray}
\label{eq: 104}
     \Delta_{1}  = \max_{\bm{z}} \, |p^{(s)}(\bm{z}) - p^{(u)}(\bm{z})|,
\end{eqnarray}
and:
\begin{eqnarray}
   \Delta_{2} = \max_{\bm{z}} \, d(l^{(u)}(\bm{z}),l^{(s)}(\bm{z})) \label{eq: 105}.
\end{eqnarray}
The first quantity $\Delta_{1}$ captures the maximum mismatch of marginal distributions between domains in latent space. 
Thus, minimizing $\Delta_{1}$ supports achieving covariate-alignment. 
On the other hand, the second quantity $\Delta_{2}$ measures the maximum mismatch of optimal classifiers between domains. 
Since stochastic classifiers are considered, enforcing a small mismatch between classifiers leads to a small discrepancy of conditional distributions $p(Y|Z)$ between domains. 
In other words, 
\tnrep{minimizing $\Delta_{2}$ supports achieving concept-alignment.}{minimizing $\Delta_{2}$ over representation functions supports achieving concept-alignment.}

\begin{corollary}
\label{cor: 1}
Under the settings of Theorem \ref{theorem: 1}, for a given representation function $f$:
\begin{eqnarray}
\epsilon^{(u)}(g, l^{(u)}) &\leq& \epsilon^{(s)}(g,l^{(s)})  \nonumber\\
&+& \Delta_{1}\int_{\bm{z}} d(g(\bm{z}),l^{(u)}(\bm{z})) \, d\bm{z} \nonumber\\
&+& \Delta_{2}. 
\end{eqnarray}
\end{corollary}

\begin{proof}
The proof directly follows by using Theorem \ref{theorem: 1}, the definitions in (\ref{eq: 104}), (\ref{eq: 105}) and the fact that $\int_{\bm{z}}p^{(s)}(\bm{z})  d\bm{z}=1$. 
\end{proof}

If the representation space is \tnrep{finite}{compact}
and the distance $d(\cdot, \cdot)$ \piedit{}{is bounded}
then
$\int_{\bm{z}} d(g(\bm{z}),l^{(u)}(\bm{z})) \, d\bm{z}$ is finite. 
Thus, minimizing the risk on seen domain $\epsilon^{(s)}(g,l^{(s)})$ together with the covariate-alignment term $\Delta_{1}$ and the concept-alignment term $\Delta_{2}$ \piedit{, it is able to guarantee that}{encourages the trained classifier $g$ to generalize well on the unseen domain.} 

\section{Practical Approach}
\label{sec: practical method}
Motivated by the theoretical results in Sec.~\ref{sec: main results}, we propose a framework that combines covariate-alignment algorithms together with concept-alignment algorithms 
for \piedit{approximately achieving domain-invariant.}{minimizing risk in the unseen domain.} 
\piedit{}{There are myriad ways of combining a given covariate-alignment algorithm with a given concept-alignment algorithm. For simplicity, we focus on minimizing the sum of the empirical prediction risk in the seen domains and a nonnegative weighted combination of the cost functions for covariate-\tnrep{}{alignment} and concept-alignment algorithms. We consider two well-studied algorithms for covariate-alignment, namely CORAL \cite{sun2016deep} and MMD \cite{li2018domain}, and one well-studied concept-alignment algorithm, namely CEM \cite{nguyen2022conditional}. We propose two DG algorithms named CORAL-CEM and MMD-CEM whose objective functions are defined as follows:}
\begin{equation}
\label{eq: practical objective function coral}
\!\!\! R_{\text{\tiny CORAL-CEM}}(f,g) := R(f,g) + \alpha L_{\text{\tiny CORAL}}(f) + \beta L_{\text{\tiny CEM}}(f,g),
\end{equation}
\begin{equation}
\label{eq: practical objective function mmd}
\!\!\! R_{\text{\tiny MMD-CEM}}(f,g) := R(f,g) + \alpha L_{\text{\tiny MMD}}(f) + \beta L_{\text{\tiny CEM}}(f,g),
\end{equation}
where the first term is the \tnrep{}{summation of the} empirical risk \tnrep{in}{from} \tnrep{the}{all} seen domains, \textit{i.e.}, \tnrep{}{summation of} $\epsilon^{(s)}(g,l^{(s)})$ over \tnrep{}{all} seen domains, the second term is the covariate-alignment term which is implemented via either the CORAL algorithm \cite{sun2016deep} or the MMD algorithm \cite{li2018domain}, the third term is the concept-alignment term which is implemented via the CEM algorithm \cite{nguyen2022conditional}, and $\alpha,\beta$ are two positive hyper-parameters that control the trade-off between these loss terms. 

\section{Experiments}
\label{sec: numerical results}
\subsection{Datasets}

In this section, we examine our proposed methods on two datasets CMNIST \cite{arjovsky2019invariant} and CS-CMNIST \cite{ahuja2021empirical}.  
To illustrate how CMNIST and CS-CMNIST datasets are generated, we use $e$ to denote domain index, $e=1,2,3$,
$X^e_g$ to denote the gray image in domain $e$, $Y^e_g$ to denote the corresponding label of $X^e_g$, $X^e$ to denote the colored image in domain $e$, and $Y^e$ to denote the corresponding label of $X^e$.
Each domain is specified by a bias parameter $\theta^e$ changing from domain to domain. We use $C^e$ to denote the color index, \piedit{}{i.e., the color assigned to $X^e_g$ to produce $X^e$.}  
Let $\oplus$ denote the XOR operator, and $\textup{Bern}(p)$ denote the Bernoulli distribution with parameter $p$. Of these three domains, two are selected as seen domains while the rest is unseen domain.

\textbf{Colored MNIST} (CMNIST) \cite{arjovsky2019invariant} is a variant of the gray MNIST handwritten digit
dataset\tnrep{ \cite{lecun1998mnist}.}{.} 
The task is to classify a colored digit (in red or blue color) into two classes: the digit is strictly less than five or the digit is greater or equal five. There are two seen domains each contains 25,000 images and one unseen domain contains 20,000 images. 
By adding the color (spurious feature) to the digit, the label is more correlated with the color than with the digit, thus, any algorithm simply aims to \tnrep{minimizing}{minimize} the training error will tend to discover the color rather than the shape of the digit and fail at the testing phase.

\begin{figure}[ht]
\centering
\includegraphics[width = 2.1 in]{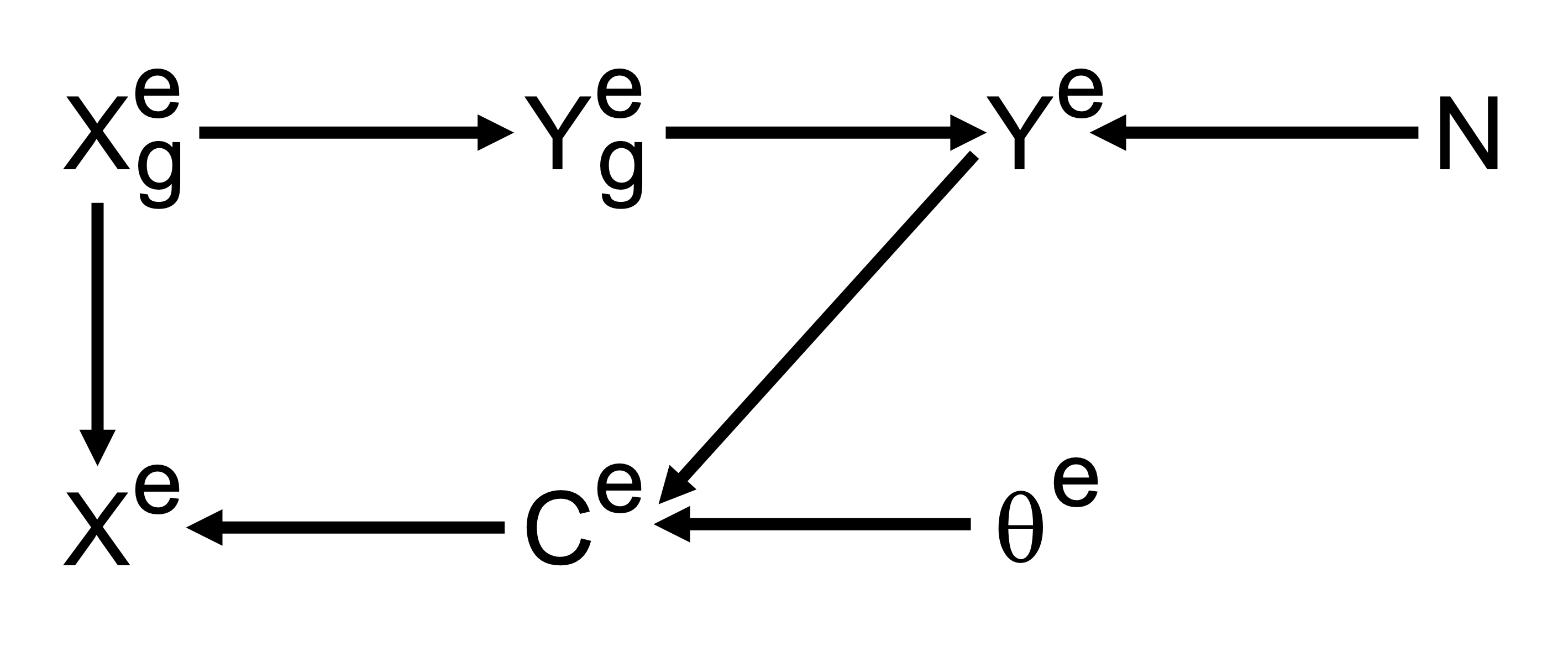}
\caption{Graphical model for CMNIST.}
\label{fig: 1}
\end{figure}
The graphical model for CMNIST dataset is illustrated in Fig. \ref{fig: 1} that contains the following steps:
\begin{enumerate}
    \item $Y_g^e \leftarrow L(X_g^e)$: 
    from 10 digits, we construct a binary classification problem by labeling $Y_g^e=0$ if the digit is less than or equal to four and \tnrep{}{labeling} $Y_g^e=1$, \tnrep{if the digit is strictly greater than four.}{otherwise.} 
    
    \item $Y^e \leftarrow Y^e_g \oplus N$, $N \sim \textup{Bern}(0.25)$: \tnrep{adding}{}noise \tnrep{}{is added} to the label $Y^e$ of the colored image. 
    
    \item $C^e  \leftarrow Y^e \oplus \theta^e$, $\theta^e \sim \textup{Bern}(p^e)$: the index color $C^e$ \tnrep{for $X_g^e$}{} is selected with a domain bias parameter $\theta^e$. 
    For \tnrep{three domains,}{} $e = 1,2,3$, \tnrep{the bias parameter}{} $p^e=0.1, 0.2, 0.9$, respectively. 
    
    \item $X^e \leftarrow T(X_g^e, C^e)$: the gray image $X^e_g$ is colored with the index color $C^e$ to produce the colored image $X^e$.  
\end{enumerate}

\textbf{Covariate-Shift-CMNIST} (CS-CMNIST) \cite{ahuja2021empirical} is a synthetic dataset derived from CMNIST. 
This dataset contains three domains (two training domains and one test domain) having 20,000 images each.
There are ten classes in CS-CMNIST where each class corresponds to a digit from one to nine. 
Each class is associated with one color that is strongly correlated with the digits in the two seen (training) domains but is weakly correlated with the digits in the unseen (test) domain. 

\begin{figure}[ht]
\centering
\includegraphics[width = 1.5 in]{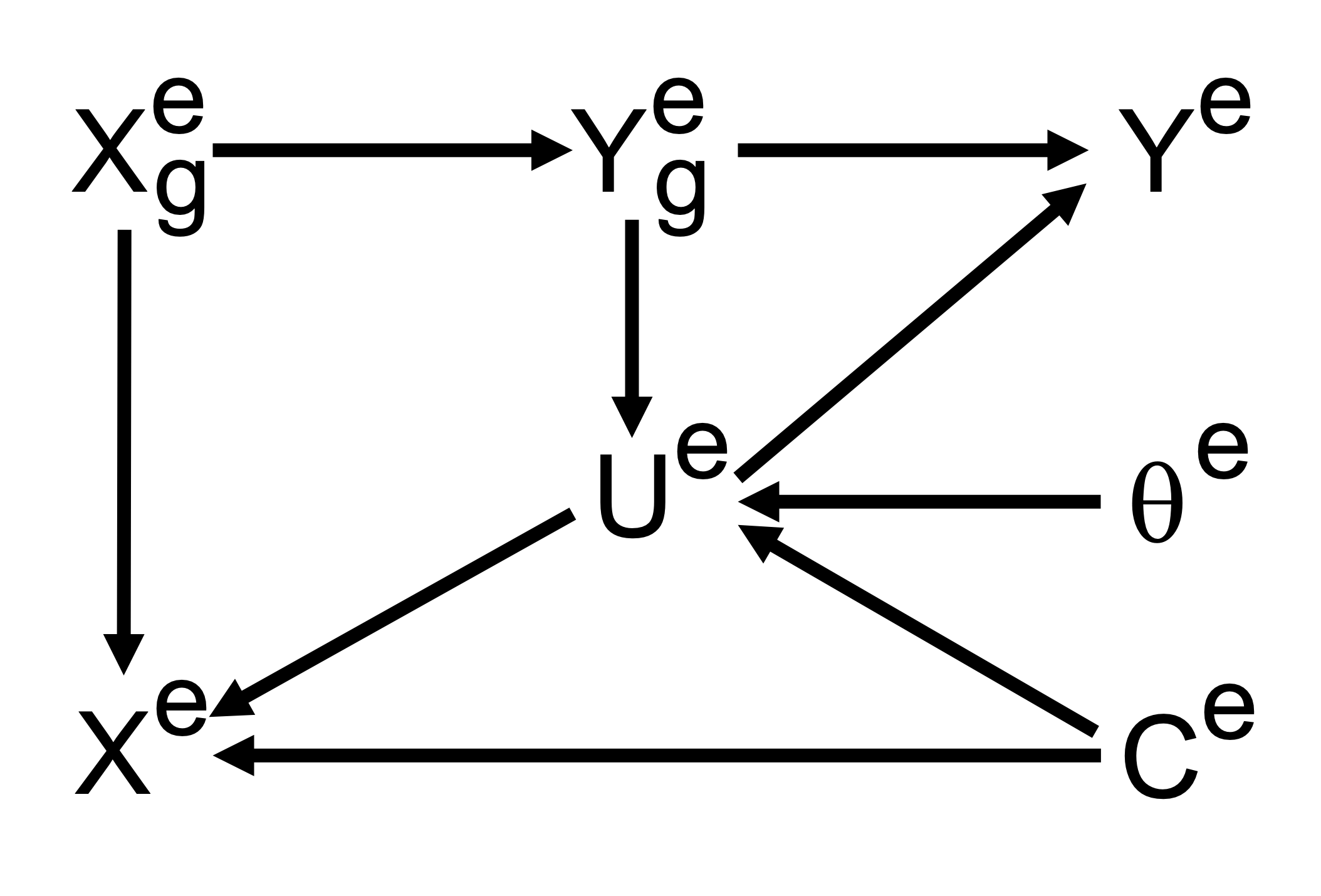}
\caption{Graphical model for CS-CMNIST.}
\label{fig: 2}
\end{figure}

The graphical model for CS-CMNIST dataset is illustrated in Fig. \ref{fig: 2} that contains the following steps:
\begin{enumerate}
    \item $Y_g^e \leftarrow L(X_g^e)$: the label of gray image is a function of gray image. There are 10 classes with labels from 0 to 9. 
    
    \item $C^e \leftarrow \textup{Uniform}(\{0,\dots,9\})$: 10 colors are picked up randomly (from color 0 to color 9).
    
    \item $U^e  \leftarrow \textup{Bern}\Big( (C^e \oplus Y^e_g) \theta^e + (1-(C^e \oplus Y^e_g))(1-\theta^e) \Big)$: \tnrep{select}{} a pair of image and its color $(X^e_g, C^e)$ \tnrep{}{is selected} with probability $\theta^e$ if $Y_g^e=C^e$, else if $Y_g^e \neq C^e$, select $(X^e_g, C^e)$ with probability $1-\theta^e$. For \tnrep{three domains,} $e = 1,2,3$, the bias parameter $\theta^e=0.1, 0.2, 0.9$, respectively. 
    
    \item $X^e \leftarrow T(X_g^e, C^e)|U^e=1$, $Y^e \leftarrow Y^e_g|U^e=1$: \tnrep{}{the} gray image $X^e_g$ is colored with color $C^e$ to produce colored image $X^e$. The colored image is labeled by $Y^e_g$.  
\end{enumerate}

From the graphical models in Fig. \ref{fig: 1} and \ref{fig: 2}, it is worth noting that while concept-shift is common in both CMNIST and CS-CMNIST, CS-CMNIST also suffers from covariate-shift (the third steps in its data generating process).

\subsection{Compared Methods}
We compare our proposed algorithms CORAL-CEM and MMD-CEM to ERM \cite{vapnik1999overview}, IRM \cite{arjovsky2019invariant},  IB-ERM \cite{ahuja2021invariance}, IB-IRM \cite{ahuja2021invariance}, MMD-IRM \cite{guo2021out}, and CEM \cite{nguyen2022conditional}. 
Since the code for MMD-IRM algorithm \tnrep{is}{was} not released by its authors, we implemented this algorithm according to the description in \cite{guo2021out}.

\subsection{Implementation}
For the CS-CMNIST dataset, we follow the learning model in \cite{ahuja2021invariance} that contains three convolutional layers with the corresponding feature dimensions of $256$, $128$, and $64$. 
We use a stochastic gradient descent optimizer for training, the batch size is set to $128$, the learning rate is $10^{-1}$ and decays after $600$ steps while the total number of steps is $2,000$. 
25 trials corresponding to 25 pairs of hyper-parameters \tnrep{$\alpha$ and $\beta$}{$(\alpha, \beta)$} are uniformly selected from $[10^{-1}, 10^4]$.

For the CMNIST dataset, we follow the setting in \cite{gulrajani2020search} and use MNIST-ConvNet with four convolutional layers as the learning model. 
The details of MNIST-ConvNet can be found in Table 7 of \cite{gulrajani2020search}. 
10 trials corresponding to 10 pairs of hyper-parameters \tnrep{$\alpha$ and $\beta$}{$(\alpha,\beta)$} are uniformly selected in $[10^{-1}, 10^4]$. For each trial, the learning rate is randomly selected in $[10^{-4.5},10^{-3.5}]$ while the batch size is randomly selected in $[2^{3},2^{9}]$.

We use the training-domain validation set tuning procedure  \cite{gulrajani2020search} for model selection, \textit{i.e.}, selecting the model with the highest validation accuracy on the validation set sampled from seen domain data. We repeat our entire experiment five times for CS-CMNIST dataset and three times for CMNIST dataset. 
Finally, the average accuracy and its standard deviation \tnrep{will be}{are} reported. Due to the limited space, the details of our implementation together with the source code can be found at  \href{https://github.com/thuan2412/Joint-covariate-alignment-and-concept-alignment-for-domain-generalization}{this link}\footnote{https://github.com/thuan2412/Joint-covariate-alignment-and-concept-alignment-for-domain-generalization}.

\subsection{Results and Discussion}
\begin{table*}[h]
\centering
\renewcommand{\arraystretch}{1.5}
\resizebox{\textwidth}{!}{\begin{tabular}{c c c c c c c c c}
\hline
\textbf{Datasets}  & \textbf{ERM \cite{vapnik1999overview}} & \textbf{IRM \cite{arjovsky2019invariant}} & \textbf{IB-ERM \cite{ahuja2021invariance}} & \textbf{IB-IRM \cite{ahuja2021invariance}} & \textbf{MMD-IRM \cite{guo2021out}} & \textbf{CEM \cite{nguyen2022conditional}} & \textbf{MMD-CEM (our)} & \textbf{CORAL-CEM (our)} \\ 
\hline
\textbf{CMNIST \cite{arjovsky2019invariant}} & 51.2 ± 0.1 & 51.4 ± 0.1 & 51.2 ± 0.3  & 52.1 ± 0.2  &   51.4 ± 0.1 & \textbf{52.5 ± 0.3} & 52.0 ± 0.2  & \textbf{52.5 ± 0.1}\\ 
\textbf{CS-CMNIST \cite{ahuja2021empirical}} & 60.3 ± 1.2  & 62.5 ± 1.1   & 71.5 ± 0.7 & 71.9 ± 0.7  &  77.2 ± 0.9 & 86.1 ± 0.8 & \textbf{90.7 ± 0.9} & 89.9 ± 0.6\\ 
\hline
\end{tabular}}
\caption{Average accuracy of the compared methods.}
\label{table: 1}
\end{table*}

Table \ref{table: 1} shows the average accuracy of the compared methods. 
As seen, our proposed algorithms outperform the rest of the tested methods with a margin of nearly 4\% for CS-CMNIST dataset. 
However, when evaluating on a more challenging CMNIST dataset, the proposed algorithms only achieve comparable or slightly better performance than other \tnrep{}{tested} methods. 
The substantial gain on CS-CMNIST can be explained by the way the data is generated. Indeed, by construction, CS-CMNIST suffers from both covariate-shift and concept-shift, therefore, it is necessary for achieving both covariate and concept alignment for better DG on CS-CMNIST. 
On the other hand, because CMNIST is designed in a way such that there exists a strong spurious correlation between colors and labels, no algorithm works on this dataset. This observation is also confirmed by the numerical results in \cite{gulrajani2020search}.

\section{Conclusions}
\label{sec: conclusion}

In this paper, by revisiting the upper bound of generalization error of the unseen domain, we motivate a domain generalization framework that combines covariate-alignment algorithms together with concept-alignment algorithms to simultaneously handle both the covariate distribution shift and the concept distribution shift\tnrep{ in DG}{}. Our numerical results show a gain on generalization performance which confirms our theoretical results. 

\small
\bibliography{example_paper}
\bibliographystyle{IEEEtran}
\end{document}